\documentclass[11pt]{article}

\usepackage[letterpaper,margin=0.99in]{geometry}

\usepackage[dvipsnames,table,xcdraw]{xcolor}
\usepackage[utf8]{inputenc}

\usepackage[T1]{fontenc}    %
\usepackage{xurl}     
\usepackage{booktabs}     %
\usepackage{nicefrac}     %

\usepackage{amsmath,amsthm,amsfonts,amssymb,epsfig,color,float,graphicx,verbatim,
enumitem}

\usepackage{bbm}
\usepackage{caption}
\usepackage{turnstile}

\usepackage[
backend=biber,
style=alphabetic,
maxbibnames=15,
maxalphanames=10,
minalphanames=6,
doi=true,
isbn=false,
url=true,
eprint=true,
backref=true,
]{biblatex}
\DefineBibliographyStrings{english}{
  backrefpage  = {},
  backrefpages = {},
}

\DeclareFieldFormat{bracketswithperiod}{\mkbibbrackets{#1}}

\renewbibmacro*{pageref}{%
  \iflistundef{pageref}
    {}
    {\printtext[bracketswithperiod]{%
       \ifnumgreater{\value{pageref}}{1}
         {\bibstring{backrefpages}}
         {\bibstring{backrefpage}}%
       \printlist[pageref][-\value{listtotal}]{pageref}}%
     }
 }
     
\renewbibmacro{in:}{}
\AtEveryBibitem{%
  \iffieldundef{doi}{}{\clearfield{url}}%
  \iffieldundef{eprint}{}{\clearfield{url}\clearfield{note}}%
}

\definecolor{lb}{RGB}{0, 100, 200}
\definecolor{green2}{RGB}{60, 120, 0}

\usepackage[colorlinks,citecolor=green2,linkcolor=lb,urlcolor=lb,bookmarks=true]{hyperref}

\usepackage[nameinlink,capitalise]{cleveref}

\usepackage{thm-restate}

\makeatletter
\ifdefined\newcounteralias
  \renewcommand*\thmt@autorefsetup{%
    \expandafter\def\csname\thmt@envname autorefname\expandafter\endcsname
      \expandafter{\thmt@thmname}%
  }%
\fi
\makeatother

\usepackage{fontawesome}

\newlist{itemizec}{itemize}{2}
\setlist[itemizec,1]{label=\faCaretRight ,wide, parsep= 0.05pt, left = 11pt}

\def\E{\mathbb E}
\def\P{\mathbb P}
\def\R{\mathbb R}

\def\eps{\varepsilon}

\newcommand{\poly}{\mathrm{poly}}

\newcommand{\cA}{\mathcal{A}}

\newcommand{\cB}{\mathcal{B}}

\newcommand{\trace}{\operatorname{tr}}

\newcommand{\cov}{\operatorname{Cov}}

\newcommand{\Var}{\operatorname{Var}}

\newcommand{\op}{\mathrm{op}}

\crefformat{equation}{(#2#1#3)}
\crefname{equation}{Equation}{Equations}
\crefname{lemma}{Lemma}{Lemmas}
\Crefname{lemma}{Lemma}{Lemmas}
\crefname{claim}{Claim}{Claims}
\crefname{fact}{Fact}{Facts}
\crefname{theorem}{Theorem}{Theorems}
\crefname{proposition}{Proposition}{Propositions}
\crefname{corollary}{Corollary}{Corollaries}
\crefname{remark}{Remark}{Remarks}
\crefname{definition}{Definition}{Definitions}
\crefname{question}{Question}{Questions}
\crefname{condition}{Condition}{Conditions}
\crefname{figure}{Figure}{Figures}

\newtheorem{theorem}{Theorem}[section]
\newtheorem{lemma}[theorem]{Lemma}

\newtheorem{corollary}[theorem]{Corollary}

\theoremstyle{definition}
\newtheorem{fact}[theorem]{Fact}
\newtheorem{definition}[theorem]{Definition}

\allowdisplaybreaks

\theoremstyle{definition}

\usepackage{color}
\definecolor{Red}{rgb}{1,0,0}
\definecolor{Blue}{rgb}{0,0,1}
\definecolor{DGreen}{rgb}{0,0.55,0}
\definecolor{Purple}{rgb}{.75,0,.25}
\definecolor{Grey}{rgb}{.5,.5,.5}

\usepackage{turnstile}
\newcommand{\sos}[3]{\sststile{#1}{#2}\,#3}
\newcommand{\EW}{\mathbb E_W}
\newcommand{\dd}{\,\mathrm{d}}
\newcommand{\dif}{\mathrm{d}}
\newcommand{\Sym}{\mathbb S^d}
\newcommand{\HS}{\mathrm{HS}}

\newcommand{\inj}{\mathrm{inj}}

\hypersetup{pdftitle={On the SoS Certifiability of Log-Concave Distributions},pdfauthor={Aleksandr Storozhenko},pdfsubject={Certifiable moments, hypercontractivity, and stochastic localization}}
\title{On the SoS Certifiability of Log-Concave Distributions}
\author{Aleksandr Storozhenko\thanks{Princeton University. \texttt{as7649@princeton.edu}.}}
\date{\today}
\begin{document}
\maketitle
\begin{abstract}
We prove that for every isotropic log-concave distribution $P$ on $\R^d$
and every even $m\ge2$, the polynomial
$(Cm)^m\|v\|_2^m - \E_{X\sim P}\langle X,v\rangle^m$
is a sum of squares, where $C>0$ is a universal constant.
This improves on the Poincar\'e-dependent bounds of Kothari and
Steinhardt~\cite{KotSteinhardt17}, recovering the optimal moment
bounds for log-concave distributions.
As an immediate corollary, we obtain computationally efficient
algorithms with dimension-free error guarantees for a wide range
of statistical estimation problems.

Our proof proceeds by using stochastic localization to decompose $P$
as an average of strongly log-concave measures, whose centered moments
admit the subgaussian certificates of~\cite{DHPT24}.
With a covariance-adapted choice of localization, we show that
a fourth-moment certificate derived from Letwin's variance inequality
for quadratic forms~\cite{Let26} suffices to control this averaging
at every even degree.
\end{abstract}%
\newpage
\section{Introduction}
\label{sec:introduction}

We study sum-of-squares (SoS) certificates for the moment tensors of
log-concave distributions. For a distribution $P$ on $\R^d$ with
mean $\mu$ and an even integer $m\ge2$, let
$M_m(P)=\E_{X\sim P}(X-\mu)^{\otimes m}$ be its centered $m$th
moment tensor. Its injective norm is given by
\begin{equation}
\|M_m(P)\|_{\inj}
= \sup_{\|v\|_2=1}\langle M_m(P),v^{\otimes m}\rangle
= \sup_{\|v\|_2=1}\E_{X\sim P}\langle X-\mu,v\rangle^m.
\end{equation}
For $m=2$, computing this norm amounts to finding the largest eigenvalue
of the covariance matrix, which can be done in polynomial time.
In contrast, even a constant-factor approximation for $m=4$ requires
superpolynomial time for general distributions under the Exponential
Time Hypothesis~\cite{BarBHKSZ12}.

The quadratic case suggests a natural relaxation. We say that $P$
has \emph{$(B_m,m)$-bounded moments} if
$\|M_m(P)\|_{\inj}\le B_m^m$, or equivalently, if the polynomial
$B_m^m\|v\|_2^m-\E_{X\sim P}\langle X-\mu,v\rangle^m$
is nonnegative on $\R^d$.
When $m=2$, the spectral theorem expresses this polynomial as a
sum of squares of linear forms.
For higher moments, a sum-of-squares representation still certifies
the bound, but its existence no longer follows from nonnegativity
alone. This leads to the notion of \emph{certifiably bounded moments}.

\begin{definition}[Certifiably Bounded Distributions~\cite{KotSte17,HopLi18}]
\label{def:cert-bdd}
Let $m\ge2$ be even and $B_m>0$. A distribution $P$ on $\R^d$ with mean
$\mu$ is \emph{$(B_m,m)$-certifiably bounded} if there are polynomials
$q_1,\ldots,q_N$ of degree at most $m/2$ such that
\[
 B_m^m\|v\|_2^m-\E_{X\sim P}\langle X-\mu,v\rangle^m
 =\sum_{j=1}^N q_j(v)^2.
\]
\end{definition}
Given the moment tensor, we can search for the certificate in
\Cref{def:cert-bdd} using a semidefinite program of size
$d^{O(m)}$~\cite{BarSte16-sos-notes,FleKP19-sos}. For fixed $m$, this
provides a tractable way to certify an upper bound on the injective norm.
The quality of the resulting bound depends on how large $B_m$
must be for such a certificate to exist.

We write $\sststile{m}{v} f\le g$ when $g-f$ is a sum of squares of degree
at most $m$. We refer the reader to \Cref{sec:sos} for background on
sum-of-squares proofs.

\paragraph{Certifiability and Algorithmic Statistics}
\label{par:cert-statistics}
The sum-of-squares \emph{proofs-to-algorithms} paradigm provides
a general method for deriving efficient algorithms from low-degree
SoS proofs of statistical identifiability~\cite{RSS18}.
For many statistical problems, these proofs use distributional
assumptions only through SoS certificates for moment
bounds~\cite{KotSte17,KotSteinhardt17,HopLi18}.
The resulting algorithms therefore apply in a \emph{black-box} fashion
to any new class of distributions admitting the required certificates.
We refer the reader to~\cite[Section 4.3]{FleKP19-sos} for a detailed
exposition of this approach to statistical estimation.

Consider, for example, robust mean estimation.
In the \emph{strong contamination model}, an adversary may inspect
$n$ independent samples from an unknown
distribution $P$ before replacing an $\eps$-fraction with arbitrary
points. Given the corrupted sample, we seek to estimate the mean
$\mu$ in Euclidean norm.
The moment bound $\|M_m(P)\|_{\inj}\le B_m^m$ suffices for
information-theoretic estimation with error
$O_m(B_m\eps^{1-1/m})$, provided $\eps$ is sufficiently small
and sufficiently many samples are available~\cite{SteCV18}.

A natural strategy is to search for a candidate clean sample that
agrees with the observations on all but an $\eps$-fraction of points
and satisfies a comparable moment bound.
The candidate and the true clean sample share at least a
$(1-2\eps)$-fraction of their points.
The moment bounds control the contribution of the remaining points
in every direction, forcing their empirical means to be close.
When these empirical moment bounds are certifiable, this identifiability
argument itself admits a low-degree SoS proof, and a semidefinite
relaxation achieves error $O_m(B_m\eps^{1-1/m})$ in polynomial time
for fixed $m$~\cite{KotSte17,HopLi18}.
For example, a fourth-moment certificate at scale $B_4$ gives
error $O(B_4\eps^{3/4})$, compared with $O(B_2\sqrt\eps)$ from
the covariance bound $\|M_2(P)\|_{\inj}\le B_2^2$ alone.
We recall this argument in \Cref{app:applications}.

The estimation error thus depends on the scale of the
\emph{certified} moment bound.
To match the information-theoretic guarantee, we therefore seek
low-degree SoS certificates with only a universal-constant loss in $B_m$.

\paragraph{Log-Concave Distributions}
Log-concavity extends the notion of convexity from sets to
probability distributions.
The class includes uniform measures on convex bodies, as well as
Gaussian, exponential, and Laplace distributions.
It is closed under taking marginals and convolutions~\cite{KV26}.

\begin{definition}[Log-Concave Distributions~\cite{vanHandel16}]
\label{def:logconcave}
A distribution is \emph{log-concave} if, on the affine span of its
support, it has a density $p(x)=e^{-V(x)}$ for a convex function
$V$ taking values in $\R\cup\{+\infty\}$.
For a positive definite matrix $B$, it is \emph{$B$-strongly
log-concave} if $V(x)-x^\top Bx/2$ is convex.
When $B=\alpha I_d$, we say \emph{$\alpha$-strongly log-concave}.
\end{definition}

We say a distribution on $\R^d$ is \emph{isotropic} if its mean is zero
and its covariance is $I_d$.
For an isotropic log-concave random vector $X$, Borell's reverse
H\"older inequality~\cite[Theorem 5.22]{LV07} gives the
dimension-free moment bound
\begin{equation}
\label{eq:ordinary-moments}
 \bigl(\E|\langle X,v\rangle|^m\bigr)^{1/m}
 \le 2m\,\E|\langle X,v\rangle|
 \le 2m\|v\|_2,
 \qquad v\in\R^d,\quad m\ge1.
\end{equation}
Thus isotropic log-concave distributions have
$(O(m),m)$-bounded moments for every even $m\ge2$.
The linear dependence on $m$ is optimal already in one dimension:
if $Z$ is an exponential random variable of mean one, then $Z-1$
is isotropic and log-concave, and
\[
 \E|Z-1|^m
 \ge m^m\P(Z\ge m+1)
 =m^m e^{-(m+1)}.
\]
This leads to the following natural question:
\begin{center}
\emph{Do the optimal moment bounds of every isotropic log-concave
distribution\\ admit sum-of-squares proofs of the same degree?}
\end{center}

\subsection{Main Results}
\label{sec:main-result}

Our main result resolves this question affirmatively.

\begin{restatable}[Certifiability of Log-Concave Distributions]
{theorem}{ThmMain}
\label{thm:main}
Every isotropic log-concave distribution $P$ on $\R^d$ is
$(Cm,m)$-certifiably bounded for all even $m\ge2$, where $C>0$
is a universal constant.
\end{restatable}

In particular, the bound in~\eqref{eq:ordinary-moments} admits
a degree-$m$ sum-of-squares proof with only a universal-constant
loss in $B_m$.

\paragraph{Poincar\'e-Based Certificates and KLS}
Recall that the \emph{Poincar\'e constant} $C_{\mathrm P}(P)$ is the
least $C$ for which
\[
 \Var_P f\le C\,\E_P\|\nabla f\|_2^2
\]
holds for every locally Lipschitz $f$ with finite variance and finite
Dirichlet energy.
We call $P$ \emph{$\sigma$-Poincar\'e} if
$C_{\mathrm P}(P)\le\sigma^2$.
Kothari and Steinhardt~\cite[Theorem 4.1]{KotSteinhardt17} showed that
a Poincar\'e inequality suffices to certify moments of every even order:
every $\sigma$-Poincar\'e distribution is
$(D_m\sigma,m)$-certifiably bounded, where $D_m$ depends only on $m$.
This gives dimension-free certificates for strongly log-concave
distributions.

For general log-concave distributions, the
Kannan--Lov\'asz--Simonovits (KLS) conjecture~\cite{KLS95} asserts that
the Poincar\'e constant is at most a universal multiple of the largest
eigenvalue of the covariance.
In isotropic position, the conjectured bound is
$C_{\mathrm P}(P)\lesssim1$; the Kothari--Steinhardt argument would
then yield \Cref{thm:main}.

The best known bound, however, is
$C_{\mathrm P}(P)\lesssim\sqrt{\log d}$ for $d\ge2$, due to
Letwin~\cite[Theorem 1.1]{Let26}.
The resulting Poincar\'e-based certificates have scale
$B_m=O_m((\log d)^{1/4})$.
Our theorem removes this dependence on the dimension and attains
the optimal scale $B_m=O(m)$, establishing this consequence of KLS
unconditionally.
We refer the reader to \Cref{sec:related} for further discussion.
\subsection{Certifiable Hypercontractivity}
\label{sec:implications}

To state our result for general covariance matrices, we use
\emph{certifiable hypercontractivity}.

\begin{definition}[Certifiable Hypercontractivity~\cite{KotSte17,KliKM18}]
\label{def:cert-hyper}
Let $m\ge2$ be an even integer and $B_m>0$. A distribution $P$ on $\R^d$
with mean $\mu$ and covariance $\Sigma$ is \emph{$(B_m,m)$-hypercontractive}
if the polynomial
\begin{align*}
 B_m^m(v^\top\Sigma v)^{m/2}
 -\E_{X\sim P}\langle X-\mu,v\rangle^m
\end{align*}
is nonnegative on $\R^d$, and \emph{$(B_m,m)$-certifiably hypercontractive}
if it is a sum of squares.
\end{definition}

After centering and whitening on the span of the centered support,
these conditions are precisely boundedness and certifiable boundedness.
Since affine maps preserve log-concavity, \Cref{thm:main} immediately
yields the following corollary.

\begin{corollary}
\label{cor:hyper}
Every log-concave distribution is $(Cm,m)$-certifiably hypercontractive for
all even $m\ge2$, with the same universal constant $C$ as in \Cref{thm:main}.
\end{corollary}

\subsection{Algorithmic Applications}
\label{sec:applications}

As an application of \Cref{cor:hyper}, we obtain robust mean
estimates in Mahalanobis norm when the covariance is unknown,
as well as covariance estimates in relative spectral norm.

For log-concave distributions, fourth-moment certificates already
give dimension-free errors $O(\eps^{3/4})$ for the mean and
$O(\sqrt\eps)$ for the covariance~\cite{KotSte17,Let26}.
The Poincar\'e-based certificates
discussed above allow higher moments to be used, but introduce a
logarithmic dependence on the dimension. Our theorem removes this
loss at every even moment order, yielding the following corollary
of~\cite[Theorem 1.2]{KotSte17}.

\begin{restatable}{corollary}{RobustEstimation}
\label{cor:robust-mean}
There exist universal constants $0<c,c'<1/2$ and $C_0>0$ such that
the following holds. Let $P$ be a log-concave distribution on $\R^d$
with mean $\mu$ and covariance $\Sigma\succ0$.
Let $0<\eps<c$ and let $m\ge4$ be even, with
$m^2\eps^{1-2/m}\le c'$.
Given an $\eps$-corrupted sample $S$ of size
$n\ge(md/\eps)^{C_0m}$ from $P$, there is an algorithm that takes
$S$, $m$, and $\eps$ as input, runs in time $(nd)^{\poly(m)}$,
and returns estimates $\widehat\mu$ and $\widehat\Sigma$ satisfying,
simultaneously with probability at least $0.9$:
\begin{itemizec}
 \item Mean estimation in Mahalanobis norm:
 $\|\Sigma^{-1/2}(\widehat\mu-\mu)\|_2
 \lesssim m\eps^{1-1/m}$.
 \item Covariance estimation in relative spectral norm:
 $(1-\delta)\Sigma\preceq\widehat\Sigma\preceq(1+\delta)\Sigma$
 for $\delta\lesssim m^2\eps^{1-2/m}$.
\end{itemizec}
\end{restatable}

For every fixed $\gamma>0$, choosing an even
$m\ge\max\{4,2/\gamma\}$ depending only on $\gamma$ gives error
$O_\gamma(\eps^{1-\gamma})$ for both tasks in polynomial time,
provided $\eps$ is sufficiently small depending only on $\gamma$.
The Mahalanobis guarantee also implies the Euclidean mean bound
$\|\widehat\mu-\mu\|_2\lesssim
m\eps^{1-1/m}\|\Sigma\|_\op^{1/2}$.
We give further discussion and a proof sketch in \Cref{app:applications}.

Beyond robust mean and covariance estimation, certifiable moment
bounds give qualitatively improved guarantees for efficient robust
estimation of higher-order moments~\cite{KotSte17,SteT21} and
robust linear regression~\cite{KliKM18,BakPra21}.
The same approach extends to clustering mixtures with separated
means~\cite{HopLi18,KotSteinhardt17}, list-decodable mean
estimation~\cite{KotSteinhardt17}, robust estimation under
differential privacy~\cite{KotMV21}, and testable learning of
halfspaces~\cite{GolKSV23}.

\subsection{Overview of Techniques}
\label{sec:overview}

Let $p_0$ be an isotropic log-concave density. For an integer
$k\ge2$, we seek a degree-$2k$ sum-of-squares proof of
\begin{align*}
 \E_{p_0}\langle X,v\rangle^{2k}\le(Ck)^{2k}\|v\|_2^{2k}.
\end{align*}
We assume for simplicity that $p_0$ has bounded support;
\Cref{sec:main-proof} removes this assumption.

\paragraph{Certifying the Fourth Moment}
We begin by recalling the fourth-moment argument of Kothari and
Steinhardt~\cite[Section 1.2.1]{KotSteinhardt17}. For isotropic
$\sigma$-Poincar\'e $X$ and every symmetric matrix $M$,
\begin{align}
 \E\langle XX^\top-I_d,M\rangle_{\HS}^2
 &=\Var(X^\top MX)\le4\sigma^2\|M\|_{\HS}^2.
 \label{eq:overview-poincare}
\end{align}
Since this inequality is quadratic in $M$, it admits a
sum-of-squares proof. Substituting $M=vv^\top$ and using
isotropy gives a degree-four certificate for
$\E\langle X,v\rangle^4\le(4\sigma^2+1)\|v\|_2^4$.

For isotropic log-concave $X$, Letwin~\cite[Theorem 1.2]{Let26} proves
\mbox{$\Var(X^\top MX)\le8\|M\|_{\HS}^2$}, removing the
dependence on $\sigma$. The same argument therefore gives
\begin{align}
 \sos{4}{v}{\E\langle X,v\rangle^4\le9\|v\|_2^4}.
 \label{eq:overview-fourth}
\end{align}

For the $2k$th moment with $k>2$, Kothari and Steinhardt use
degree-$k$ polynomials $f$ satisfying
$\E[\nabla^j f(X)]=0$ for $0\le j<k$.
Applying Poincar\'e $k-2$ times reduces the problem to the
centered quadratic entries of $\nabla^{k-2}f$. Letwin then gives
\begin{align*}
 \E f(X)^2
 \le\sigma^{2k-4}\E\|\nabla^{k-2}f(X)\|_{\HS}^2
 \le2\sigma^{2k-4}\|\nabla^k f\|_{\HS}^2.
\end{align*}
The preceding $k-2$ applications of Poincar\'e leave the factor
$\sigma^{2k-4}$. Thus obtaining dimension-free certificates by
this route still requires the Poincar\'e bound conjectured by KLS.

\paragraph{A Random Decomposition}
We instead start with the moment certificates available for
strongly log-concave distributions. A $t$-strongly log-concave
distribution is $t^{-1/2}$-subgaussian~\cite[Theorem 1.1]{Har04},
so its centered $2k$th moment has a degree-$2k$ certificate with
bound $(Kk/t)^k\|v\|_2^{2k}$ \cite{DHPT24}, where $K$ is a universal constant.

To use these certificates for $p_0$, we seek a decomposition
into strongly log-concave distributions. Eldan's stochastic
localization~\cite{Eld13} provides one through small random
reweightings. We write $\mu_t$ and $A_t$ for the mean and
covariance of the current density $p_t$. To first order,
a step of length $h$ takes the form
\begin{align*}
 p_{t+h}(x)\approx p_t(x)\left(
 1+\sqrt h\,\langle C_t^{1/2}(x-\mu_t),z\rangle\right),
 \qquad z\sim N(0,I_d),
\end{align*}
where $z$ is a fresh Gaussian vector and the matrix $C_t\succ0$
controls the strength of the reweighting in each direction.
Centering at $\mu_t$ preserves total mass, while averaging over
$z$ leaves the density unchanged. Letting the step size tend
to zero leads to the stochastic differential equation
\begin{align*}
 \dd p_t(x)=p_t(x)
 \langle C_t^{1/2}(x-\mu_t),\dd W_t\rangle,
\end{align*}
where $W_t$ is standard Brownian motion. The resulting density
has the form
\begin{align*}
 p_t(x)\propto p_0(x)
 \exp\!\left(c_t^\top x-\tfrac12x^\top B_tx\right),
 \qquad B_t=\int_0^t C_s\dd s,
\end{align*}
where $c_t$ is a random vector. Thus each $p_t$ is
$B_t$-strongly log-concave, while the mean-zero reweightings give
$\EW[\E_{p_t}f]=\E_{p_0}f$ for every bounded $f$, where
$\EW$ averages over the process.
At any positive time $T$, averaging the moments of $p_T$
therefore recovers those of $p_0$.

\paragraph{Choosing the Control}
By~\cite{DHPT24}, the centered $2k$th moment of $p_T$ has a
degree-$2k$ certificate with bound $(Kk)^k(v^\top B_T^{-1}v)^k$.
Splitting $X=(X-\mu_T)+\mu_T$ and averaging gives
\begin{align}
 \sststile{2k}{v}\quad \E_{p_0}\langle X,v\rangle^{2k}
 &\le2^{2k-1}\Bigl((Kk)^k\EW[(v^\top B_T^{-1}v)^k]
       +\EW\langle\mu_T,v\rangle^{2k}\Bigr).
 \label{eq:overview-split}
\end{align}

A natural first attempt is to take $C_t=I_d$.
Then $B_T=TI_d$, so the bound on the centered
moments in~\eqref{eq:overview-split} is $(Kk/T)^k\|v\|_2^{2k}$.
To obtain the desired bound, we therefore need $T\gtrsim1/k$.
The means are harder to control. A step of length $h$ changes
the mean, to first order, by $\sqrt h\,A_tz$.
Its variance per unit time in direction $v$ is therefore
\begin{align*}
 \E_z\langle A_tz,v\rangle^2
 =\|A_tv\|_2^2=v^\top A_t^2v.
\end{align*}
Bounding this by $\|A_t\|_{\op}(v^\top A_tv)$ introduces the
largest covariance eigenvalue. However, for some log-concave
distributions, $\EW\|A_t\|_{\op}$ already grows with $d$ at time
$t\asymp1/\log d$~\cite[Proposition 65]{KL24}.
When $k\ll\log d$, this happens before the time $T\asymp1/k$
needed for the centered moments, ruling out a universal bound
throughout $[0,T]$.

To avoid this dependence on $\|A_t\|_{\op}$, we choose $C_t=A_t^{-1}$.
The mean's variance per unit time in direction $v$ is then $v^\top A_tv$.
Now $B_T$ is random, but we show in
the proof of \Cref{lem:precision} that
\begin{align*}
 v^\top B_T^{-1}v
 =v^\top\left(\int_0^T A_t^{-1}\dd t\right)^{-1}v
 \le T^{-2}\int_0^T v^\top A_tv\dd t.
\end{align*}
The same directional variances therefore control both the
centered moments and the means. By \Cref{lem:means,lem:precision},
bounding the two terms in~\eqref{eq:overview-split} reduces to
finding a dimension-free upper bound on
\begin{align*}
 G_t(v)=\EW[(v^\top A_tv)^k],
\end{align*}
with a degree-$2k$ certificate, up to time $T\asymp1/k$.

\paragraph{Lifting the Fourth-Moment Certificate}
We write the directional variance as
\begin{align*}
 q_t(v)=v^\top A_tv
 =\E_{p_t}\langle X,v\rangle^2-\langle\mu_t,v\rangle^2.
\end{align*}
It\^o's formula (see~\cite[Section 4]{KV26} for an introduction)
states that, for a smooth function $\phi$,
\begin{align*}
 \dd\phi(q_t(v))
 =\phi'(q_t(v))\dd q_t(v)
  +\tfrac12\phi''(q_t(v))\dd[q(v)]_t.
\end{align*}
The first term is the usual chain rule. The second involves
the quadratic variation $[q(v)]_t$ and accounts for the squared
Brownian increments.

Since $\E_{p_t}\langle X,v\rangle^2$ is a martingale, applying
It\^o's formula to the squared mean gives
\begin{align*}
 \dd q_t(v)
 &=\underbrace{\dd\E_{p_t}\langle X,v\rangle^2
   -2\langle\mu_t,v\rangle\dd\langle\mu_t,v\rangle}
   _{\text{martingale terms}}
   -q_t(v)\dd t.
\end{align*}
We now apply the formula with $\phi(s)=s^k$.
The martingale terms are multiplied by $kq_t(v)^{k-1}$, which
depends only on the process up to time $t$. Since each new
Brownian increment has conditional mean zero, these terms
vanish in expectation; bounded support ensures square-integrability.
We obtain
\begin{align*}
 G_t'(v)
 =-kG_t(v)+\binom{k}{2}
   \EW\!\left[q_t(v)^{k-2}\frac{\dd}{\dd t}[q(v)]_t\right].
\end{align*}
To bound the second term, we use Letwin's estimate. Since $p_t$
remains log-concave, we show in \Cref{lem:fluctuations} that
\begin{align}
 \frac{\dd}{\dd t}[q(v)]_t\le8q_t(v)^2.
 \label{eq:overview-variation}
\end{align}
Substituting this bound gives
\begin{align*}
 G_t'(v)
 &\le-kG_t(v)+8\binom{k}{2}\EW[q_t(v)^k]
 \le4k(k-1)G_t(v).
\end{align*}

Our key observation is that the same calculation admits a
degree-$2k$ sum-of-squares proof. Indeed, we show in
\Cref{lem:fluctuations} that the inequality
in~\eqref{eq:overview-variation} has a degree-four certificate.
Its multiplier $q_t^{k-2}$ in It\^o's formula is itself a sum
of squares, since $q_t$ is a nonnegative quadratic form.
Multiplying the certificate by $q_t^{k-2}$ therefore gives
degree $2k$. More explicitly,
\begin{align}
 4k(k-1)G_t(v)-G_t'(v)
 &=kG_t(v)\nonumber\\
 &\quad+\binom{k}{2}\EW\!\left[
 q_t(v)^{k-2}
 \left(8q_t(v)^2-\frac{\dd}{\dd t}[q(v)]_t\right)\right].
 \label{eq:overview-growth}
\end{align}
Both terms on the right are sums of squares: the second
averages the multiplied certificate, while $G_t=\EW[q_t^k]$
averages powers of a nonnegative quadratic form.

Multiplying by $e^{-4k(k-1)t}$ and integrating preserves the
certificate. Since $G_0(v)=\|v\|_2^{2k}$, we obtain
\begin{align}
 \sos{2k}{v}{G_T(v)\le e^{4k(k-1)T}\|v\|_2^{2k}}.
 \label{eq:overview-potential}
\end{align}
%
\subsection{Related and Prior Work}
\label{sec:related}

\paragraph{Certifiable Boundedness of Probability Distributions}
Certifiable moment bounds have extensive applications in algorithmic
statistics, as \hyperref[par:cert-statistics]{discussed above}; for a broader
treatment of algorithmic robust statistics, see~\cite{DiaKan22-book}.
Early results established certifiability for product and rotationally
invariant subgaussian distributions~\cite{KotSte17,HopLi18}.
Kothari and Steinhardt~\cite{KotSteinhardt17} obtained certificates
for distributions satisfying a Poincar\'e inequality, with bounds
depending on the Poincar\'e constant.
These classes can also be extended using natural closure properties
of certifiable boundedness~\cite{DHPT24}.
Diakonikolas, Hopkins, Pensia, and Tiegel~\cite{DHPT24} proved that
every subgaussian distribution is certifiably subgaussian, and asked
whether analogous certificates exist for subexponential distributions.
\Cref{thm:main} answers this question for log-concave distributions.

\paragraph{Injective Norms}
Our theorem gives sum-of-squares upper bounds on the injective norms
of moment tensors of log-concave distributions.
Such bounds, or equivalently bounds on homogeneous polynomials over
the unit sphere, have been studied for both worst-case
tensors~\cite{doherty2012convergence,bhattiprolu2017weak} and random
tensors~\cite{hopkins2015tensor,raghavendra2017strongly,potechin2020machinery}.
For even $m$, the injective norm of $\sum_{i=1}^n a_i^{\otimes m}$
equals $\|A\|_{2\to m}^m$, where $A$ has rows $a_1,\ldots,a_n$.
Thus certification of moment tensors is closely related to
approximation of hypercontractive matrix norms~\cite{BarBHKSZ12,HT26}.

Hopkins and Tiegel~\cite{HT26} show that every
$(B_m,m)$-bounded distribution admits degree-$O(m)$ SoS certificates
at scale $d^{1/4-1/(2m)}B_m$, with the same guarantee for
hypercontractive moment bounds.
Their result requires only a bound at the single even order $m$.
For isotropic log-concave distributions, it gives the scale
$O(md^{1/4-1/(2m)})$, whereas \Cref{thm:main} gives $O(m)$.

Under the Small Set Expansion Hypothesis, constant-factor
approximation of moment-tensor injective norms remains hard even
for distributions with subgaussian moment bounds up to any fixed
order~\cite{BarBHKSZ12,HopLi19}.
Consequently, under this hypothesis, our theorem cannot be extended
to all distributions satisfying only a fixed number of moment bounds.

\paragraph{Relation to KLS}
Our theorem establishes unconditionally a consequence of KLS:
the dimension-free moment certificates that would follow from
the Poincar\'e-based argument of~\cite{KotSteinhardt17}.
We are not aware of any improvement to the known KLS bounds
that follows from our result.
For isotropic log-concave measures, KLS is equivalent to a
dimension-free variance bound for every $1$-Lipschitz
function~\cite{Mil09}.
Our certificates concern moments of linear projections:
increasing their degree controls higher powers of
$\langle X,v\rangle$, but does not by itself give variance
bounds for general nonlinear functions.

\paragraph{Organization}
\Cref{sec:preliminaries} collects the sum-of-squares background and
the moment and variance bounds used in the proof.
\Cref{sec:transfer} proves the localization transfer theorem, and
\Cref{sec:main-proof} uses it to prove \Cref{thm:main}.
Further details on sum-of-squares proofs and stochastic localization
are given in \Cref{app:sos,app:localization}, respectively.
\Cref{app:applications} explains the applications to robust mean and
covariance estimation.%
%
\section{Preliminaries}
\label{sec:preliminaries}

\paragraph{Notation}
We write $\|\cdot\|_{\op}$ and $\|\cdot\|_{\HS}$ for the operator and
Hilbert--Schmidt norms, and $M\preceq N$ when $N-M$ is positive semidefinite.
The space $\Sym$ of symmetric matrices has inner product
$\langle M,N\rangle_{\HS}=\trace(MN)$. We use $\lesssim$ to suppress
universal constants. For the localization process, $\E_{p_t}$ integrates
under $p_t$, whereas $\EW$ averages over the Brownian motion.
Expectations of polynomials are taken coefficientwise.

\subsection{Sum-of-Squares Proofs and Certificates}
\label{sec:sos}

We recall the sum-of-squares proof notation used in our argument and refer
the reader to~\cite{BarSte16-sos-notes,FleKP19-sos} for further background.

Let $\cA=\{r_1(v)\ge0,\ldots,r_s(v)\ge0\}$ be a (possibly empty) system of polynomial inequalities in formal variables $v$. Let $p$ be another polynomial in $v$. We say that $\cA$ implies $p\ge0$ at degree $t$, denoted by $\cA\sststile{t}{v} p\ge0$, if we can write
\[
 p(v)=\sum_{S\subseteq[s]}b_S(v)\prod_{i\in S}r_i(v),
\]
where each $b_S$ is a sum of squares and each summand has degree at most $t$. When $\cB=\{r'_1(v)\ge0,\ldots,r'_{s'}(v)\ge0\}$ is another system of polynomial inequalities, we write $\cA\sststile{t}{v}\cB$ if $\cA\sststile{t}{v} r'_j\ge0$ for all $j\in[s']$. We write $\cA\sststile{t}{v} p\le p'$ if $\cA\sststile{t}{v} p'-p\ge0$. When $\cA$ is empty, we omit it. The facts used below are collected in \Cref{app:sos}.

\subsection{Moment and Variance Bounds}
\label{sec:strong-certificates}

Strongly log-concave distributions have subgaussian tails. We will use
sum-of-squares certificates for their moments, so we begin with the
moment definition of subgaussianity.

\begin{definition}[Subgaussian Distributions~\cite{vanHandel16}]
\label{def:subgaussian}
Let $s>0$. A distribution $P$ on $\R^d$ with mean $\mu$ is
\emph{$s$-subgaussian} if, for every $r\ge1$ and $v\in\R^d$,
\[
 \bigl(\E_{X\sim P}|\langle X-\mu,v\rangle|^r\bigr)^{1/r}
 \le C_0s\sqrt r\,\|v\|_2,
\]
where $C_0\ge2$ is a fixed universal constant.
\end{definition}

Diakonikolas, Hopkins, Pensia, and Tiegel showed that these moment bounds
admit sum-of-squares proofs with only a constant-factor loss.
\begin{theorem}[{\cite[Theorem 1.6]{DHPT24}}]
\label{thm:subgaussian}
There is a universal constant $C_{\mathrm{sg}}>0$ such that every
$s$-subgaussian distribution is $(C_{\mathrm{sg}}s\sqrt m,m)$-certifiably
bounded for every even $m\ge2$.
\end{theorem}

Every $1$-strongly log-concave distribution is
$1$-subgaussian~\cite[Theorem 1.1]{Har04}, so its moments are certifiable
at scale $O(\sqrt m)$. We will need the following version, which allows
strong log-concavity to be specified by a positive-definite matrix $B$.
\begin{corollary}[Strongly Log-Concave Moment Certificates]
\label{cor:strong-moments}
There is a universal constant $K\ge2$ such that, for every
$B$-strongly log-concave distribution $Q$ with $B\succ0$ and mean $\mu$,
and every integer $k\ge1$,
\begin{equation}
\label{eq:strong-moments}
 \sos{2k}{v}{\E_{X\sim Q}\langle X-\mu,v\rangle^{2k}
       \le(Kk)^k(v^\top B^{-1}v)^k}.
\end{equation}
\end{corollary}
\begin{proof}
Let $X\sim Q$ and set $Y=B^{1/2}(X-\mu)$. Then $Y$ is $1$-strongly
log-concave and hence $1$-subgaussian. By \Cref{thm:subgaussian},
\begin{align*}
 \E\langle X-\mu,v\rangle^{2k}
 &=\E\langle Y,B^{-1/2}v\rangle^{2k}\\*
 &\le(C_{\mathrm{sg}}\sqrt{2k})^{2k}\|B^{-1/2}v\|_2^{2k}\\*
 &=(2C_{\mathrm{sg}}^2k)^k(v^\top B^{-1}v)^k.
\end{align*}
The substitution $u=B^{-1/2}v$ is linear, so this is a degree-$2k$
sum-of-squares proof. Taking
$K=\max\{2,2C_{\mathrm{sg}}^2\}$ proves the corollary.
\end{proof}

The moment certificates above will apply to the strongly log-concave
components of our decomposition. To control their covariance fluctuations,
we use a different input: Letwin's variance bound for quadratic forms under
isotropic log-concave distributions. In \Cref{sec:fourth-certificates},
we will turn this bound into a fourth-degree certificate.
\begin{theorem}[Letwin, {\cite[Theorem 1.2]{Let26}}]
\label{thm:letwin}
Let $Y\in\R^d$ be isotropic and log-concave. Then, for every symmetric
$M\in\R^{d\times d}$,
\begin{equation}
\label{eq:letwin}
 \Var(Y^\top MY)\le8\|M\|_{\HS}^2.
\end{equation}
\end{theorem}

%
\section{Moment Certificates via Stochastic Localization}
\label{sec:transfer}

The following theorem transfers moment certificates from strongly
log-concave distributions to isotropic log-concave distributions with an
$O(\sqrt{k})$ loss at degree $2k$.

\begin{theorem}[Localization Transfer]
\label{thm:transfer}
Let $k\ge2$ be an integer and $L_k\ge1$. Suppose that every
$1$-strongly log-concave distribution is $(L_k,2k)$-certifiably bounded.
Then every compactly supported isotropic log-concave density $p_0$ satisfies
\begin{equation}
\label{eq:transfer}
 \sos{2k}{v}{\E_{p_0}\langle X,v\rangle^{2k}
 \le 2^{2k-1}e^{4k(k-1)T}
       \left(L_k^{2k}T^{-k}+1\right)\|v\|_2^{2k}}
\end{equation}
for all $T>0$. In particular, $p_0$ is
$(4\sqrt{ek}\,L_k,2k)$-certifiably bounded.
\end{theorem}

\subsection{Fourth-Degree Certificates}
\label{sec:fourth-certificates}

We first obtain a fourth-moment certificate from Letwin's variance bound.

\begin{lemma}[Fourth-Moment Certificate]
\label{lem:fourth}
Let $Y\in\R^d$ be isotropic and log-concave. Then
\begin{equation}
\label{eq:fourth}
 \sos{4}{u}{\Var(\langle Y,u\rangle^2)\le8\|u\|_2^4}.
\end{equation}
In particular, the law of $Y$ is $(\sqrt{3},4)$-certifiably bounded.
\end{lemma}
\begin{proof}
By \Cref{thm:letwin}, the quadratic form
$8\|M\|_{\HS}^2-\Var(Y^\top MY)$ is nonnegative in the entries of $M$,
and hence a sum of squares of linear forms. Substituting $M=uu^\top$
gives~\eqref{eq:fourth}.
The fourth-moment bound follows from isotropy:
\begin{align*}
 \sststile{4}{u}\quad \E\langle Y,u\rangle^4
 &=\Var(\langle Y,u\rangle^2)+\|u\|_2^4
 \le9\|u\|_2^4.\qedhere
\end{align*}
\end{proof}

To bound the third-moment contraction, we use the following projection
inequality. It requires only isotropy; its proof is deferred to
\Cref{app:projection}.

\begin{restatable}[Projection Bound]{lemma}{LemProjection}
\label{lem:projection}
Let $Y\in\R^d$ be isotropic with finite fourth moments. Then
\begin{equation}
\label{eq:projection}
 \sos{4}{u}{\left\|\E[Y\langle Y,u\rangle^2]\right\|_2^2
                  \le\Var(\langle Y,u\rangle^2)}.
\end{equation}
\end{restatable}

Combining \Cref{lem:fourth,lem:projection}, we obtain the following corollary.
\begin{corollary}[Third-Moment Contraction]
\label{cor:third-contraction}
Let $Y\in\R^d$ be isotropic and log-concave. Then
\begin{equation}
\label{eq:third-contraction}
 \sos{4}{u}{\left\|\E[Y\langle Y,u\rangle^2]\right\|_2^2
                  \le8\|u\|_2^4}.
\end{equation}
\end{corollary}

\subsection{The Random Reweighting}
\label{sec:process}

Stochastic localization expresses a log-concave density as an average of
strongly log-concave densities. We will use \Cref{cor:third-contraction}
to control their covariance fluctuations. Let $p_0$ be an isotropic
log-concave density supported in $\{x:\|x\|_2\le R\}$, and let $W_t$ be
a standard Brownian motion in $\R^d$.

\begin{definition}[Stochastic Localization~\cite{Eld13}]
\label{def:controlled-localization}
For a continuous adapted matrix process $C_t\succ0$, define
\begin{equation}
 \begin{aligned}
 c_0=0,\qquad &\dif c_t=C_t\mu_t\dd t+C_t^{1/2}\dif W_t,\\
 B_0=0,\qquad &\dif B_t=C_t\dd t,
 \end{aligned}
 \label{eq:param-sde}
\end{equation}
where $\mu_t$ and $A_t$ are the mean and covariance of the density
\begin{equation}
 p_t(x)=\frac{e^{c_t^\top x-\frac12x^\top B_tx}p_0(x)}
 {\int e^{c_t^\top y-\frac12y^\top B_ty}p_0(y)\dd y},
 \qquad \mu_t=\E_{p_t}X,\qquad A_t=\cov_{p_t}(X).
 \label{eq:localized-density}
\end{equation}
\end{definition}

Throughout, we take the inverse-covariance control
\begin{align*}
 C_t=A_t^{-1},\qquad B_t=\int_0^t A_s^{-1}\dd s.
\end{align*}
By~\cite[Section 2 and Lemma 2.4]{Eld13}, this process has a unique solution
for all finite times, with $A_t\succ0$. For $t>0$, the density $p_t$ is
$B_t$-strongly log-concave. The density evolution and averaging properties
are proved in \Cref{app:localization}.

The infinitesimal reweighting is linear, so the stochastic part of the
covariance equation involves the third centered-moment tensor of $p_t$.
Its slices are
\begin{align}
 H_{t,i}
 &=\E_{p_t}[(X-\mu_t)(X-\mu_t)^\top(X-\mu_t)_i],
 \qquad i\in[d].
 \label{eq:third-slices}
\end{align}
The following identity is proved in \Cref{app:moment-equations}.

\begin{restatable}[Covariance Evolution, {\cite[Lemma 4.5]{KV26}}]{fact}{FactCovariance}
\label{fact:covariance-evolution}
The covariance satisfies
\begin{equation}
 \dif A_t=\sum_{i=1}^d H_{t,i}(A_t^{-1/2}\dif W_t)_i-A_t\dd t.
 \label{eq:covariance-sde}
\end{equation}
\end{restatable}

Fix $v\in\R^d$. We express the covariance fluctuations in isotropic
coordinates. For $X\sim p_t$, set
\begin{align*}
 Y=A_t^{-1/2}(X-\mu_t),\qquad u=A_t^{1/2}v.
\end{align*}
Then $Y$ is isotropic and log-concave, and
$\langle Y,u\rangle=\langle X-\mu_t,v\rangle$.
Taking the quadratic form of~\eqref{eq:covariance-sde} gives
\begin{align}
 \dif(v^\top A_tv)
 &=\left\langle\E[Y\langle Y,u\rangle^2],\dif W_t\right\rangle
   -v^\top A_tv\dd t.
 \label{eq:directional-sde}
\end{align}
The squared norm of the vector multiplying $\dif W_t$ is the
quadratic-variation rate:
\begin{align}
 \frac{\dif}{\dif t}[v^\top A_{\cdot}v]_t
 &=\left\|\E[Y\langle Y,u\rangle^2]\right\|_2^2.
 \label{eq:directional-variation}
\end{align}
The drift contributes a nonpositive term to the expected powers of
$v^\top A_tv$. To bound their growth using It\^o's formula, it therefore
suffices to control~\eqref{eq:directional-variation}, which is exactly the
quantity in \Cref{cor:third-contraction}.

\begin{lemma}[Certifiable Covariance Fluctuations]
\label{lem:fluctuations}
For each realized density $p_t$,
\begin{equation}
\label{eq:fluctuations}
 \sos{4}{v}{\frac{\dif}{\dif t}[v^\top A_{\cdot}v]_t
       \le8(v^\top A_tv)^2}.
\end{equation}
\end{lemma}
\begin{proof}
Apply \Cref{cor:third-contraction} to $Y$ and substitute $u=A_t^{1/2}v$.
This linear substitution preserves the proof degree, and
$\|u\|_2^4=(v^\top A_tv)^2$, so~\eqref{eq:directional-variation} gives the claim.
\end{proof}

\subsection{Certifying the Covariance Tensor Powers}
\label{sec:potential}

Fix $k\ge2$ and let
\begin{align}
 G_t(v)=\EW[(v^\top A_tv)^k],
 \qquad \beta=4k(k-1).\nonumber
\end{align}
Bounded support and \Cref{lem:fluctuations} justify taking expectations
in the It\^o identities below; see \Cref{app:localization}.

\begin{lemma}[Covariance Potential]
\label{lem:potential}
For every $T\ge0$,
\begin{align}
\label{eq:potential}
 \sos{2k}{v}{G_T(v)=\EW[(v^\top A_Tv)^k]
       \le e^{\beta T}\|v\|_2^{2k}}.
\end{align}
\end{lemma}
\begin{proof}
Taking expectations in It\^o's formula for $(v^\top A_tv)^k$ and using
\eqref{eq:directional-sde}, we obtain
\begin{align}
 G_t'(v)
 &=-kG_t(v)+\binom{k}{2}\EW\left[
    (v^\top A_tv)^{k-2}
    \frac{\dif}{\dif t}[v^\top A_{\cdot}v]_t\right]\nonumber\\*
 &\le(\beta-k)G_t(v)
 \le\beta G_t(v)
 \label{eq:potential-derivative}
\end{align}
for almost every $t$. The stochastic integral has expectation zero, and
the first inequality follows from \Cref{lem:fluctuations}. Multiplying
its degree-four certificate by $(v^\top A_tv)^{k-2}$ and averaging gives
a degree-$2k$ certificate for $G_t'\le\beta G_t$.

Since $G_0(v)=\|v\|_2^{2k}$, integration gives
\begin{equation}
 e^{\beta T}\|v\|_2^{2k}-G_T(v)
 =\int_0^T e^{\beta(T-t)}
       \bigl(\beta G_t(v)-G_t'(v)\bigr)\dd t.
 \label{eq:potential-certificate}
\end{equation}
The right-hand side is a sum of squares of degree $2k$ by
\Cref{lem:averaging}, which proves the claim.
\end{proof}

\subsection{Bounds on the Mean and \texorpdfstring{$B_T^{-1}$}{B(T) Inverse}}
\label{sec:transfer-bounds}

The mean equation bounds the moments of $\langle\mu_T,v\rangle$ in
terms of $G_t$. The identity $B_T=\int_0^T A_t^{-1}\dd t$ gives the
corresponding bound for $v^\top B_T^{-1}v$.

\begin{lemma}[Moments of the Mean]
\label{lem:means}
For every $T\ge0$,
\begin{align}
\label{eq:means}
 \sos{2k}{v}{\EW\langle\mu_T,v\rangle^{2k}
       \le e^{\beta T}\|v\|_2^{2k}}.
\end{align}
\end{lemma}
\begin{proof}
Let $F_t(v)=\EW\langle\mu_t,v\rangle^{2k}$. By
\Cref{fact:mean-evolution}, $\dif\mu_t=A_t^{1/2}\dif W_t$.
Thus It\^o's formula and Young's inequality give
\begin{align}
 F_t'(v)
 &=k(2k-1)\EW[\langle\mu_t,v\rangle^{2k-2}(v^\top A_tv)]\nonumber\\*
 &\le\alpha F_t(v)+(2k-1)G_t(v),
 \label{eq:means-derivative}
\end{align}
where $\alpha=(k-1)(2k-1)$. Since $F_0=0$, integrating and using
\Cref{lem:potential} yields
\begin{align}
 F_T(v)
 &\le(2k-1)\int_0^T e^{\alpha(T-t)}G_t(v)\dd t\nonumber\\*
 &\le(2k-1)\int_0^T e^{\alpha(T-t)+\beta t}\dd t\,\|v\|_2^{2k}\nonumber\\*
 &=\frac{2k-1}{\beta-\alpha}
       \bigl(e^{\beta T}-e^{\alpha T}\bigr)\|v\|_2^{2k}.
 \label{eq:means-comparison}
\end{align}
As $\beta-\alpha=(k-1)(2k+1)\ge2k-1$, this is at most
$e^{\beta T}\|v\|_2^{2k}$. Young's inequality and the integrating-factor
argument have degree-$2k$ certificates by
\Cref{fact:young,lem:integration}.
\end{proof}

\begin{lemma}[Moments of $v^\top B_T^{-1}v$]
\label{lem:precision}
For every $T>0$,
\begin{align}
\label{eq:precision}
 \sos{2k}{v}{\EW[(v^\top B_T^{-1}v)^k]
       \le T^{-k}e^{\beta T}\|v\|_2^{2k}}.
\end{align}
\end{lemma}
\begin{proof}
We first bound $v^\top B_T^{-1}v$ by
$T^{-2}\int_0^T v^\top A_tv\dd t$ for each realization. Indeed, the
identity $B_T=\int_0^T A_t^{-1}\dd t$ gives
\begin{equation}
 \int_0^T v^\top A_tv\dd t-T^2v^\top B_T^{-1}v
 =\int_0^T
   \left\|A_t^{1/2}v-TA_t^{-1/2}B_T^{-1}v\right\|_2^2\dd t
 \ge0.
 \label{eq:precision-square}
\end{equation}
By Jensen's inequality and \Cref{lem:potential}, it follows that
\begin{align}
 \EW[(v^\top B_T^{-1}v)^k]
 &\le T^{-2k}\EW\left[\left(\int_0^T v^\top A_tv\dd t\right)^k\right]
 \nonumber\\*
 &\le T^{-k-1}\int_0^T G_t(v)\dd t\nonumber\\*
 &\le T^{-k-1}\int_0^T e^{\beta t}\dd t\,\|v\|_2^{2k}\nonumber\\*
 &\le T^{-k}e^{\beta T}\|v\|_2^{2k}.
 \label{eq:precision-comparison}
\end{align}
Together with the quadratic certificate in~\eqref{eq:precision-square},
\Cref{fact:jensen,lem:averaging} show that these inequalities have a
degree-$2k$ certificate.
\end{proof}

\subsection{Averaging the Certificates}

\begin{proof}[Proof of \Cref{thm:transfer}]
Let $p_t$ be the process of \Cref{def:controlled-localization} with
$C_t=A_t^{-1}$, and fix $T>0$.
For each realization, $B_T^{1/2}(X-\mu_T)$ under $p_T$ is
$1$-strongly log-concave by \Cref{lem:density}, and therefore
$(L_k,2k)$-certifiably bounded.
Its moment certificate, evaluated at $u=B_T^{-1/2}v$, gives
\begin{align}
 \E_{p_T}\langle X-\mu_T,v\rangle^{2k}
 \le L_k^{2k}(v^\top B_T^{-1}v)^k
 \label{eq:transfer-centered-moments}
\end{align}
with a degree-$2k$ certificate. By \Cref{lem:martingale} and the
decomposition $X=(X-\mu_T)+\mu_T$, we have
\begin{align}
 \E_{p_0}\langle X,v\rangle^{2k}
 &=\EW[\E_{p_T}\langle X,v\rangle^{2k}]\nonumber\\
 &\le2^{2k-1}\EW\left[
    \E_{p_T}\langle X-\mu_T,v\rangle^{2k}
       +\langle\mu_T,v\rangle^{2k}\right]\nonumber\\
 &\le2^{2k-1}\left(
    L_k^{2k}\EW[(v^\top B_T^{-1}v)^k]
       +\EW\langle\mu_T,v\rangle^{2k}\right)\nonumber\\
 &\le2^{2k-1}e^{\beta T}
       \left(L_k^{2k}T^{-k}+1\right)\|v\|_2^{2k}.
 \label{eq:transfer-average}
\end{align}
Here the first inequality is the even-power triangle inequality, the
second is~\eqref{eq:transfer-centered-moments}, and the last uses
\Cref{lem:means,lem:precision}. Together they give the degree-$2k$
certificate~\eqref{eq:transfer}.

Finally, choose $T=\frac{1}{4(k-1)}$, which minimizes $e^{\beta T}T^{-k}$.
Since $L_k\ge1$, we have
\begin{align*}
 2^{2k-1}e^k\bigl([4(k-1)]^kL_k^{2k}+1\bigr)
 &\le[16e(k-1)]^kL_k^{2k}\\*
 &\le\bigl(4\sqrt{ek}\,L_k\bigr)^{2k},
\end{align*}
which proves the last assertion.
\end{proof}

%
\section{Certifiable Moments of Log-Concave Distributions}
\label{sec:main-proof}

\ThmMain*
\begin{proof}
Let $X\sim P$ and fix an even $m\ge4$; the case $m=2$ follows from isotropy.
By \Cref{lem:truncation}, there are compactly supported isotropic
log-concave vectors $Y_R$ whose moments converge to those of $X$.
Consequently, for any fixed $C$,
\begin{align*}
 (Cm)^m\|v\|_2^m-\E\langle Y_R,v\rangle^m
 &\longrightarrow
 (Cm)^m\|v\|_2^m-\E\langle X,v\rangle^m
\end{align*}
coefficientwise. Since the degree-$m$ sum-of-squares cone is closed
(\Cref{lem:averaging}), it suffices to prove the theorem for compactly
supported distributions, with a constant independent of the support.

Assume now that $P$ is compactly supported and write $m=2k$.
By \Cref{cor:strong-moments}, every $1$-strongly log-concave distribution
is $(\sqrt{Kk},2k)$-certifiably bounded, where $K$ is universal.
Taking $L_k=\sqrt{Kk}$ in \Cref{thm:transfer}, we obtain
\begin{align*}
 \sststile{m}{v}\quad \E\langle X,v\rangle^m
 &\le \bigl(4\sqrt{ek}\,\sqrt{Kk}\bigr)^{2k}\|v\|_2^m\\*
 &=\bigl(2\sqrt{eK}\,m\bigr)^m\|v\|_2^m.
\end{align*}
Thus $C=2\sqrt{eK}$ gives the required degree-$m$ certificate.
\end{proof}

%
\section*{Acknowledgments}
We thank Pravesh K.~Kothari, \c{S}tefan Tudose, and Santosh S.~Vempala
for the many helpful discussions.

\paragraph{AI Disclosure}
The author began working on this problem in February 2026.
Our original proof was obtained by adapting Guan's argument~\cite{Gua24},
replacing the spectral potential with moments of the directional variances
and using Letwin's bound on the third-moment tensor~\cite{Let26}
to establish dimension-free bounds on their growth within sum of squares.
Recently, we found that GPT-6 Pro could produce a proof using
the same input from Letwin and a similar analysis
\href{https://chatgpt.com/s/t_6aad080ac6688191907e800813f5b637}{in response to a single prompt}.
Following this, the author refined the original proof using a simpler
choice of localization control suggested by GPT-6 Pro, obtaining the
unified argument presented here.
GPT-6 was also used to draft and edit initial versions of portions
of this paper, which were then heavily edited and rewritten by the author.%
\printbibliography
\appendix
\section{Sum-of-Squares Toolkit}
\label[appendix]{app:sos}

We record the polynomial inequalities and closure properties used to average moment certificates. All certificates are unconstrained unless stated otherwise.

\subsection{Basic Sum-of-Squares Facts}

\begin{fact}[Spectral Certificates]
\label{fact:spectral}
Let $M\preceq N$ be real symmetric matrices and let $z(v)$ be a vector of polynomials of degree at most $r$. Then
\[
 \sos{2r}{v}{z(v)^\top Mz(v)\le z(v)^\top Nz(v)}.
\]
\end{fact}

\begin{fact}[Young's Inequality]
\label{fact:young}
Let $a(v),b(v)$ be nonnegative quadratic forms and let $k\ge2$ be an integer. Then
\[
 \sos{2k}{v}{ka^{k-1}b\le(k-1)a^k+b^k}.
\]
\end{fact}

\begin{fact}[Jensen's Inequality]
\label{fact:jensen}
Let $k\ge1$ be an integer and let $a_\omega(v)$ be a measurable family of nonnegative quadratic forms on a probability space. If the coefficients of $a_\omega^k$ are integrable, then
\[
 \sos{2k}{v}{(\E_\omega a_\omega(v))^k
                     \le\E_\omega[a_\omega(v)^k]}.
\]
\end{fact}
For $k\ge2$, this follows by averaging Young's inequality with
$a=\E_\omega a_\omega$ and $b=a_\omega$, using \Cref{lem:averaging};
the case $k=1$ is an identity.

\begin{fact}[Even-Power Triangle Inequality]
\label{fact:triangle}
Let $\ell_1(v),\ell_2(v)$ be linear forms. For every integer $k\ge1$,
\[
 \sos{2k}{v}{(\ell_1+\ell_2)^{2k}
             \le2^{2k-1}(\ell_1^{2k}+\ell_2^{2k})}.
\]
\end{fact}

\subsection{Projection onto Linear Functions}
\label[appendix]{app:projection}

\LemProjection*
\begin{proof}
Put $b(u)=\E[Y\langle Y,u\rangle^2]$. Since $\E Y=0$ and
$\E YY^\top=I_d$, expanding the square gives
\begin{align*}
 &\E\left[\left(\langle Y,u\rangle^2-\|u\|_2^2
                    -\langle Y,b(u)\rangle\right)^2\right]\\
 &\qquad=\Var(\langle Y,u\rangle^2)
       -2\left\langle\E[Y(\langle Y,u\rangle^2-\|u\|_2^2)],b(u)\right\rangle
       +b(u)^\top\E[YY^\top]b(u)\\
 &\qquad=\Var(\langle Y,u\rangle^2)-\|b(u)\|_2^2.
\end{align*}
For each fixed value of $Y$, the integrand is the square of a quadratic
form in $u$. Its expectation is therefore a degree-four sum of squares,
which proves the claim.
\end{proof}

\subsection{Averaging and Limits}

\begin{lemma}[Closedness and Averaging]
\label{lem:averaging}
In a fixed number of variables, the cone of sum-of-squares polynomials of degree at most $2k$ is closed under coefficientwise limits. If $p_\omega(v)$ belongs to this cone for almost every $\omega$ and its coefficients are measurable and integrable with respect to a finite positive measure $\nu$, then
\[
 \int p_\omega(v)\dd\nu(\omega)
\]
also belongs to the cone.
\end{lemma}
\begin{proof}
We first prove closedness. Suppose $p_j\to p$ coefficientwise, where each
$p_j$ is a sum of squares of degree at most $2k$. Let $z(v)$ be the vector
of monomials of degree at most $k$, and choose $Q_j\succeq0$ such that
$p_j(v)=z(v)^\top Q_jz(v)$.

For a standard Gaussian vector $G$, the matrix
$H=\E[z(G)z(G)^\top]$ is positive definite, since the monomials in $z$ are
linearly independent in $L_2(G)$. Consequently,
\[
 \lambda_{\min}(H)\trace Q_j
 \le\trace(HQ_j)
 =\E p_j(G)\longrightarrow\E p(G).
\]
The matrices $Q_j$ therefore have bounded trace and admit a convergent
subsequence, with limit $Q\succeq0$. It follows that $p(v)=z(v)^\top Qz(v)$.

The cone is convex as well as closed, so the averaging assertion follows
by separation. Indeed, for every linear functional $\ell$ nonnegative on
the cone, coefficientwise integrability gives
\[
 \ell\!\left(\int p_\omega\dd\nu(\omega)\right)
 =\int\ell(p_\omega)\dd\nu(\omega)\ge0.\qedhere
\]
\end{proof}

\subsection{Integrating Polynomial Differential Inequalities}

\begin{lemma}[Positive Integrating Factors]
\label{lem:integration}
Let $F_t(v)$ have degree at most $2k$ and coefficients that are absolutely continuous on $[0,T]$. Let $H_t(v)$ have degree at most $2k$ and integrable coefficients, and let $a\in\R$. If
\[
 \sos{2k}{v}{F_t'(v)\le aF_t(v)+H_t(v)}
\]
for almost every $t$, then
\[
 \sos{2k}{v}{F_T(v)\le e^{aT}F_0(v)
               +\int_0^T e^{a(T-t)}H_t(v)\dd t}.
\]
\end{lemma}
\begin{proof}
The integrating-factor identity gives
\[
 e^{aT}F_0+\int_0^T e^{a(T-t)}H_t\dd t-F_T
 =\int_0^T e^{a(T-t)}(aF_t+H_t-F_t')\dd t.
\]
By assumption, $aF_t+H_t-F_t'$ is a sum of squares of degree at most $2k$ for almost every $t$. Averaging with the positive weight $e^{a(T-t)}$ proves the claim by \Cref{lem:averaging}.
\end{proof}

%
\section{Stochastic Localization}
\label[appendix]{app:localization}

We record the remaining properties of the process in
\Cref{def:controlled-localization} and prove
\Cref{fact:covariance-evolution}. Throughout, $C_t=A_t^{-1}$.

\subsection{Density Evolution and Averaging}
\label{app:density-evolution}

\begin{fact}[Density Evolution, {\cite[Lemma 4.4]{KV26}}]
\label{lem:density}
The density $p_t$ satisfies
\begin{equation}
 \dif p_t(x)=p_t(x)\langle x-\mu_t,A_t^{-1/2}\dif W_t\rangle.
 \label{eq:density-sde}
\end{equation}
For $t>0$, it is $B_t$-strongly log-concave.
\end{fact}
\begin{proof}
By~\eqref{eq:param-sde}, $B_t$ has finite variation. Applying It\^o's
formula to the logarithm of the normalizing integral
in~\eqref{eq:localized-density} gives
\begin{align}
 &\dif\log\int e^{c_t^\top y-\frac12y^\top B_ty}p_0(y)\dd y\nonumber\\*
 &\qquad=\langle\mu_t,\dif c_t\rangle
       -\frac12\E_{p_t}[X^\top A_t^{-1}X]\dd t
       +\frac12\trace(A_tA_t^{-1})\dd t\nonumber\\*
 &\qquad=\langle\mu_t,\dif c_t\rangle
       -\frac12\mu_t^\top A_t^{-1}\mu_t\dd t,\nonumber
\end{align}
where the second equality follows from the definition of the covariance.
For $x$ with $p_0(x)>0$, the density formula therefore gives
\begin{align}
 \dif\log p_t(x)
 &=\langle x-\mu_t,\dif c_t\rangle
   -\frac12\bigl(x^\top A_t^{-1}x-\mu_t^\top A_t^{-1}\mu_t\bigr)\dd t\nonumber\\*
 &=\langle x-\mu_t,A_t^{-1/2}\dif W_t\rangle
   -\frac12\|A_t^{-1/2}(x-\mu_t)\|_2^2\dd t.
 \label{eq:log-density-sde}
\end{align}
Applying It\^o's formula to the exponential gives
\begin{align}
 \dif p_t(x)
 &=p_t(x)\dif\log p_t(x)
   +\frac12p_t(x)\|A_t^{-1/2}(x-\mu_t)\|_2^2\dd t\nonumber\\*
 &=p_t(x)\langle x-\mu_t,A_t^{-1/2}\dif W_t\rangle.\nonumber
\end{align}
If $p_0(x)=0$, then $p_t(x)=0$ for all $t$, so the identity still holds.
Finally, the density formula gives
\begin{align}
 e^{\frac12x^\top B_tx}p_t(x)&\propto e^{c_t^\top x}p_0(x).\nonumber
\end{align}
The right-hand side is log-concave, so $p_t$ is $B_t$-strongly log-concave.
\end{proof}

The density equation also gives the averaging identity used in the
proof of \Cref{thm:transfer}.

\begin{fact}[Preservation of the Expected Measure]
\label{lem:martingale}
For every measurable $f$ bounded on the support of $p_0$, the process
$\E_{p_t}f$ is a square-integrable martingale. In particular,
\begin{equation}
 \EW[\E_{p_t}f]=\E_{p_0}f.
 \label{eq:measure-average}
\end{equation}
\end{fact}
\begin{proof}
Integrating~\eqref{eq:density-sde} against $f$ gives
\begin{equation}
 \dif\E_{p_t}f
 =\left\langle A_t^{-1/2}\E_{p_t}[(X-\mu_t)f(X)],\dif W_t\right\rangle.
 \label{eq:test-sde}
\end{equation}
Stochastic Fubini applies after stopping where $c_t$, $B_t$, and $A_t^{-1}$
are bounded. Thus $\E_{p_t}f$ is a local martingale, and the bound
$|\E_{p_t}f|\le\|f\|_{L_\infty(p_0)}$ makes it a square-integrable
martingale.
\end{proof}

\subsection{Mean and Covariance Equations}
\label{app:moment-equations}

\begin{fact}[Mean Evolution]
\label{fact:mean-evolution}
The mean satisfies
\begin{align}
 \dif\mu_t&=A_t^{1/2}\dif W_t.
 \label{eq:mean-sde}
\end{align}
\end{fact}
\begin{proof}
Apply~\eqref{eq:test-sde} to the coordinate functions. Since
$\E_{p_t}(X-\mu_t)=0$, we obtain
\begin{align}
 \dif\mu_t
 &=\E_{p_t}[X(X-\mu_t)^\top]A_t^{-1/2}\dif W_t\nonumber\\*
 &=\E_{p_t}[(X-\mu_t)(X-\mu_t)^\top]A_t^{-1/2}\dif W_t\nonumber\\*
 &=A_t^{1/2}\dif W_t.\nonumber\qedhere
\end{align}
\end{proof}

Taking the quadratic variation of~\eqref{eq:mean-sde} in a fixed direction
gives the following corollary.
\begin{corollary}[Quadratic Variation of the Mean]
\label{cor:mean-variation}
For every fixed $v\in\R^d$,
\begin{align}
 \frac{\dif}{\dif t}[\langle\mu_{\cdot},v\rangle]_t&=v^\top A_tv.
 \label{eq:mean-variation}
\end{align}
\end{corollary}

\FactCovariance*
\begin{proof}
Apply It\^o's product rule to
$A_t=\E_{p_t}[XX^\top]-\mu_t\mu_t^\top$.
Using \Cref{fact:mean-evolution}, we obtain
\begin{align}
 \dif A_t
 &=\int xx^\top\dif p_t(x)\dd x
      -(\dif\mu_t)\mu_t^\top-\mu_t(\dif\mu_t)^\top-A_t\dd t\nonumber\\*
 &=\int (x-\mu_t)(x-\mu_t)^\top
      \langle x-\mu_t,A_t^{-1/2}\dif W_t\rangle p_t(x)\dd x-A_t\dd t\nonumber\\*
 &=\sum_{i=1}^d H_{t,i}(A_t^{-1/2}\dif W_t)_i-A_t\dd t.\nonumber
\end{align}
The second equality follows by centering and~\eqref{eq:density-sde};
the third uses~\eqref{eq:third-slices}.
\end{proof}

Finally, every $p_t$ has the same bounded support as $p_0$, so
\begin{align}
 \|\mu_t\|_2\le R,\qquad A_t\preceq R^2I_d,\qquad
 B_T^{-1}\preceq\frac{R^2}{T}I_d\quad(T>0).\nonumber
\end{align}
For fixed $k$ and $v$, these estimates and \Cref{lem:fluctuations} bound
the drift and Brownian coefficients in the moment equations of
\Cref{sec:transfer}. The stochastic integrals on finite intervals are
therefore square-integrable with mean zero, and expectation commutes with
time integration coefficientwise.

\subsection{Truncation and Affine Normalization}

\begin{lemma}[Compactly Supported Approximation]
\label{lem:truncation}
Let $X$ be an isotropic log-concave random vector in $\R^d$. There are
compactly supported isotropic log-concave random vectors $Y_R$ such that
\[
 \E f(Y_R)\longrightarrow\E f(X)\qquad(R\to\infty)
\]
for every polynomial $f:\R^d\to\R$.
\end{lemma}
\begin{proof}
Let $X_R$ have the law of $X$ conditioned on $\|X\|_2\le R$, with mean
$a_R$ and covariance $\Sigma_R$. Since $X$ has moments of all
orders~\cite{Bor74}, dominated convergence gives
\begin{align*}
 \E f(X_R)
 &=\frac{\E[f(X)\mathbf1_{\{\|X\|_2\le R\}}]}
          {\P(\|X\|_2\le R)}
 \longrightarrow\E f(X)
\end{align*}
for every polynomial $f$. In particular, $a_R\to0$ and $\Sigma_R\to I_d$.
For all sufficiently large $R$, we may therefore set
\[
 Y_R=\Sigma_R^{-1/2}(X_R-a_R).
\]
Conditioning on a convex set and applying an affine map preserve
log-concavity, so $Y_R$ is isotropic, log-concave, and compactly supported.
Finally, $\Sigma_R^{-1/2}\to I_d$, so for each fixed $f$ the coefficients of
$f(\Sigma_R^{-1/2}(x-a_R))$ converge to those of $f(x)$. Together with
the moment convergence of $X_R$, this proves the claim.
\end{proof}

%
\section{Outlier-Robust Mean and Covariance Estimation}
\label{app:applications}

We give a proof sketch of \Cref{cor:robust-mean}, extending the
identifiability argument from the introduction to control both
means and covariances. We first recall the strong contamination
model~\cite{DiaKKLMS16-focs,DiaKan22-book}.

\begin{definition}[Strong Contamination Model]
\label{def:strong-contamination}
Let $P$ be a distribution on $\R^d$ and let $0<\eps<1$.
A clean sample of $n$ independent points is drawn from $P$.
After inspecting the entire sample, a computationally unbounded
adversary may replace at most $\eps n$ of its points by arbitrary
vectors in $\R^d$. The resulting sample is called an
\emph{$\eps$-corrupted sample} from $P$.
The estimator observes this sample but is not told which points
were replaced.
\end{definition}

For a distribution with mean $\mu$ and covariance $\Sigma\succ0$,
we seek estimates $\widehat\mu$ and $\widehat\Sigma$ with small
error in Mahalanobis norm and relative spectral norm, respectively:
\[
 \|\Sigma^{-1/2}(\widehat\mu-\mu)\|_2
 \qquad\text{and}\qquad
 \|\Sigma^{-1/2}(\widehat\Sigma-\Sigma)\Sigma^{-1/2}\|_\op.
\]

\RobustEstimation*

For fixed $m$, the sample complexity and
running time are polynomial in $d$ and $1/\eps$.
Choosing an even $m=\Theta(\log(1/\eps))$ gives the optimal
contamination errors $O(\eps\log(1/\eps))$ and
$O(\eps\log^2(1/\eps))$, with quasipolynomial sample complexity
and running time, for sufficiently small $\eps$.

\begin{proof}[Proof sketch]
Following~\cite{KotSte17}, we first give an information-theoretic
criterion and then sketch how to turn it into an efficient estimator
using SoS. Given the observations $y_1,\ldots,y_n$, we require a
candidate clean sample $x'_1,\ldots,x'_n$, with empirical mean
$\mu'$ and covariance $\Sigma'$, to satisfy
\begin{equation}
\label{eq:robust-constraints}
 \begin{gathered}
  |\{i:x'_i\ne y_i\}|\le\eps n,\\[2pt]
  \frac1n\sum_{i=1}^n\langle x'_i-\mu',v\rangle^\ell
  \le(C\ell)^\ell(v^\top\Sigma'v)^{\ell/2}
 \end{gathered}
\end{equation}
for every $v\in\R^d$ and every even $2\le\ell\le m$, where $C$
is a sufficiently large universal constant.
The moment bounds are relative to $\Sigma'$, so the constraints
can be imposed when $\Sigma$ is unknown.
The estimator returns $(\mu',\Sigma')$ for any such sample.

Let $x_1,\ldots,x_n$ be the clean sample, with empirical mean
$\mu_0$ and covariance $\Sigma_0\succ0$.
By \Cref{cor:hyper} and the sampling result
in~\cite[Lemma 5.5]{KotSte17}, this sample satisfies the constraints
with high probability. On this event, fix any candidate sample
satisfying the constraints.

Both samples agree with the observations
at all but $\eps n$ indices, so they differ only on a set
$B=\{i:x'_i\ne x_i\}$ of relative size $\rho=|B|/n\le2\eps$.
Centering the identity
$\mu'-\mu_0=n^{-1}\sum_{i\in B}(x'_i-x_i)$ and applying
H\"older's inequality gives
\begin{equation}
\label{eq:robust-mean-overlap}
 \begin{aligned}
  (1-\rho)|\langle\mu'-\mu_0,v\rangle|
  &=\left|\frac1n\sum_{i\in B}
    \langle(x'_i-\mu')-(x_i-\mu_0),v\rangle\right|\\
  &\le Cm\rho^{1-1/m}
    \bigl(\sqrt{v^\top\Sigma'v}+\sqrt{v^\top\Sigma_0v}\bigr)\\
  &\lesssim m\eps^{1-1/m}
    \sqrt{v^\top(\Sigma'+\Sigma_0)v}.
 \end{aligned}
\end{equation}
The mean bound still involves $\Sigma'$. To express the error
relative to $\Sigma_0$ alone, we also need to compare the covariances.
We apply the same overlap argument to second moments, using
pairwise differences to remove the dependence on the empirical means:
\[
 v^\top\Sigma_0v
 =\frac1{2n^2}\sum_{i,j}\langle x_i-x_j,v\rangle^2,
\]
and its analogue for $\Sigma'$.
Only pairs with an index in $B$ contribute to the difference
between the two sums. There are at most $2\rho n^2\le4\eps n^2$
such pairs.
The moment bounds and the even-power triangle inequality give
\[
 \frac1{n^2}\sum_{i,j}\langle x_i-x_j,v\rangle^m
 \le(2Cm)^m(v^\top\Sigma_0v)^{m/2},
\]
and similarly for $x'_1,\ldots,x'_n$.
H\"older's inequality with exponent $m/2$, applied to the squared
pairwise differences on these pairs, therefore gives
\begin{equation}
\label{eq:robust-cov-overlap}
 |v^\top(\Sigma'-\Sigma_0)v|
 \lesssim m^2\eps^{1-2/m}\,v^\top(\Sigma'+\Sigma_0)v.
\end{equation}
Since $m^2\eps^{1-2/m}$ is sufficiently small, writing
$\Sigma'+\Sigma_0=2\Sigma_0+(\Sigma'-\Sigma_0)$
in~\eqref{eq:robust-cov-overlap} and rearranging gives
\begin{equation}
\label{eq:robust-cov-comparison}
 (1-\delta)\Sigma_0\preceq\Sigma'\preceq(1+\delta)\Sigma_0,
 \qquad \delta\lesssim m^2\eps^{1-2/m}.
\end{equation}
In particular, \eqref{eq:robust-cov-comparison} gives
$\Sigma'+\Sigma_0\preceq(2+\delta)\Sigma_0$.
Substituting this bound into~\eqref{eq:robust-mean-overlap}
and using $1-\rho\ge1-2\eps$, we get
\begin{equation}
\label{eq:robust-mean-comparison}
 \|\Sigma_0^{-1/2}(\mu'-\mu_0)\|_2
 =\sup_{v^\top\Sigma_0v=1}|\langle\mu'-\mu_0,v\rangle|
 \lesssim m\eps^{1-1/m}.
\end{equation}

To obtain an efficient estimator, we strengthen
\eqref{eq:robust-constraints} by requiring degree-$\ell$ SoS
certificates in $v$ for the moment bounds. Introducing their
coefficients as auxiliary variables gives a finite polynomial
encoding of these bounds; see~\cite[Section 4.3.4]{FleKP19-sos}.
By \Cref{cor:hyper} and the same sampling result, the clean sample
still satisfies these constraints. The comparison above admits
a low-degree SoS proof.
The semidefinite relaxation of Kothari and
Steurer~\cite[Theorem 1.2]{KotSte17} therefore yields estimates
satisfying~\eqref{eq:robust-cov-comparison}
and~\eqref{eq:robust-mean-comparison} in the claimed running time.
At the stated sample size, the sampling errors between
$(\mu_0,\Sigma_0)$ and $(\mu,\Sigma)$ only change the constants
in these bounds, which proves the claim.
\end{proof}
\end{document}